\documentclass{llncs}
\usepackage{times}

\usepackage{ifthen}
\newcommand{\isdraft}{\boolean{true}} 
\renewcommand{\isdraft}{\boolean{false}} 
\ifthenelse{\isdraft}{
    \usepackage{showkeys}
    \usepackage{xcolor}
    \usepackage[colorlinks=true,linkcolor=blue]{hyperref}
}{}
\newcommand{\markupdraft}[2]{
    \ifthenelse{\equal{#1}{display}}{#2}{}
    \ifthenelse{\equal{#1}{color}}{\color{#2}}{}
}
\newcommand{\notecolored}[3][]{\markupdraft{display}{{\color{#2}\noindent[Note (#1): #3]}}}
\newcommand{\newcolored}[3][]{{\markupdraft{color}{#2}#3}
    \ifthenelse{\equal{#1}{}}{}{\markupdraft{display}{{\color{yellow!70!black}[#1]}}}} 

\providecommand{\del}[2][]{{\markupdraft{display}{{\color{red!20!yellow}[rmed: "#2"[#1]]}}}} 
\providecommand{\new}[2][]{\newcolored[#1]{blue}{#2}}
\providecommand{\nnew}[2][]{\newcolored[#1]{red}{#2}}
\providecommand{\rem}[2][]{\notecolored[#1]{green}{#2}}    

\ifthenelse{\isdraft}{}{\renewcommand{\markupdraft}[2]{}}
\newcommand{\niko}[1]{\rem[Niko]{\color{brown}#1}}
\newcommand{\yohe}[1]{\rem[Youhei]{#1}}
\newcommand{\anne}[1]{\rem[Anne]{\color{orange}#1}}

\usepackage{epsfig}
\usepackage{amsmath}
\usepackage{amssymb}
\usepackage{latexsym}
\usepackage{graphicx}
\usepackage{cite}

\DeclareMathOperator{\Tr}{Tr}

\DeclareMathOperator{\vect}{vec} 

\providecommand{\abs}[1]{\left\lvert#1\right\rvert} 
\providecommand{\norm}[1]{\left\lVert#1\right\rVert} 
\providecommand{\tp}{\mathrm{T}} 
\providecommand{\dx}{\mathrm{d}x} 
\providecommand{\dz}{\mathrm{d}z}
\providecommand{\rmd}{\mathrm{d}}

\renewcommand{\geq}{\geqslant}
\renewcommand{\leq}{\leqslant}

\renewcommand{\epsilon}{\varepsilon}
\renewcommand{\phi}{\varphi}

\providecommand{\btheta}{\bar \theta} 
\providecommand{\stheta}{\theta^{*}} 
\providecommand{\thetad}{\theta'} 
\providecommand{\PP}[1]{P_{#1}} 
\providecommand{\pd}[1]{p_{#1}} 
\providecommand{\PPZ}{\PP{d}} 
\providecommand{\Gauss}{\mathcal{N}} 
\providecommand{\mm}{m} 
\providecommand{\vv}{v} 
\providecommand{\CC}{C} 

\providecommand{\FF}{J} 
\providecommand{\Ft}{\FF_{\theta}} 

\providecommand{\ww}{w} 
\providecommand{\ddt}[1]{\frac{\mathrm{d}#1}{\mathrm{dt}}}
\providecommand{\VV}{V} 
\providecommand{\xstar}{\ensuremath{\mathrm{x^{*}}}}
\providecommand{\xbar}{\ensuremath{\mathrm{\bar{x}}}}

\providecommand{\E}{\mathbb{E}}
\providecommand{\R}{\mathbb{R}}

\providecommand{\FM}{\mathcal{I}}
\providecommand{\FIM}{\FM_{\theta}}

\providecommand{\eqesigo}{\eqref{eq:esigo} }
\providecommand{\gvv}{F_{\vv}}
\providecommand{\gmm}{F_{\mm}}

\providecommand{\gtt}{F_{\theta}}

\providecommand{\Wf}{W_{\theta}^{f}}
\providecommand{\dd}{{d_\theta}} 
\providecommand{\tVV}{\tilde \VV}
\providecommand{\gttz}{F_\theta(\theta, z)}
\providecommand{\flow}{\phi}
\providecommand{\asm}[1]{\textnormal{(#1)}}
\providecommand{\vf}{F}
\providecommand{\sif}{\mathcal{G}}
\providecommand{\Ctwo}{\mathcal{C}^{2}}
\providecommand{\leb}{\mu_\mathrm{Leb}}

\begin{document}

\title{Convergence of the Continuous Time Trajectories of Isotropic Evolution Strategies on Monotonic $\Ctwo$-composite Functions
}
\author{Youhei Akimoto \and Anne Auger \and Nikolaus Hansen}
\institute{TAO Team, INRIA Saclay-Ile-de-France, LRI, Paris Sud University, France\\
\email{\{Youhei.Akimoto, Anne.Auger, Nikolaus.Hansen\}@lri.fr}
}

\maketitle              

\begin{abstract}
The \textit{Information-Geometric Optimization (IGO)} has been introduced as a unified framework for stochastic search algorithms. Given a parametrized family of probability distributions on the search space, the IGO turns an arbitrary optimization problem on the search space into an optimization problem on the parameter space of the probability distribution family and defines a natural gradient ascent on this space. From the natural gradients defined over the entire parameter space we obtain continuous time trajectories which are the solutions of an ordinary differential equation (ODE).
Via discretization, the IGO naturally defines an iterated gradient ascent algorithm. Depending on the chosen distribution family, the IGO recovers several known algorithms such as the pure rank-$\mu$ update CMA-ES. Consequently, the continuous time IGO-trajectory can be viewed as an idealization of the original algorithm. 

In this paper we study the continuous time trajectories of the IGO given the family of isotropic Gaussian distributions. 
These trajectories are a deterministic continuous time model of the underlying evolution strategy in the limit for population size to infinity and change rates to zero.
On functions that are the composite of a monotone and a convex-quadratic function, we prove the global convergence of the solution of the ODE towards the global optimum. We extend this result to composites of monotone and twice continuously differentiable functions and prove local convergence towards local optima.
\end{abstract}



%
%
%

\section{Introduction}

Evolution Strategies (ESs) are stochastic search algorithms for numerical optimization. In ESs, candidate solutions are sampled using a Gaussian distribution parametrized by a mean vector and a covariance matrix. In state-of-the art ESs, those parameters are iteratively adapted using the ranking of the candidate solutions w.r.t. the objective function. Consequently, ESs are invariant to applying a monotonic transformation to the objective function. Adaptive ES algorithms are successfully applied in practice and there is ample empirical evidence that they converge linearly towards a local optimum of the objective function on a wide class of functions. However, their theoretical analysis even on simple functions is difficult as the state of the algorithm is given by both the mean vector and the covariance matrix that have a stochastic dynamic that needs to be simultaneously controlled. Their linear convergence to local optima is so far only proven for functions that are composite of a monotonic transformation with a convex quadratic function---hence function with a single optimum---for rather simple search algorithms compared to the covariance matrix adaptation evolution strategy (CMA-ES) that is considered as the state-of-the-art ES \cite{Auger2005tcs,Jagerskupper:2006ur,Jagerskupper:2006cf,Jagerskupper2007tcs}. In this paper, instead of analyzing the exact stochastic dynamic of the algorithms, we consider the deterministic time continuous model underlying adaptive ESs that follows from the Information-Geometric Optimization (IGO) setting recently introduced \cite{Arnold2011arxiv}.

The Information-Geometric Optimization is a unified framework for randomized search algorithms. Given a family of probability distributions parametrized by $\theta \in \Theta$, the original objective function, $f$, is transformed to a fitness function $\Ft$ defined on $\Theta$. The IGO algorithm defined on $\Theta$ performs a natural gradient ascent aiming at maximizing $\Ft$. For the family of Gaussian distributions, the IGO algorithm recovers the pure rank-$\mu$ update CMA-ES \cite{Hansen2003ec}, for the family of Bernoulli distributions, PBIL \cite{Chandy:1995ux} is recovered. When the step-size for the gradient ascent algorithm (that corresponds to a learning rate in CMA-ES and PBIL) goes to zero, we obtain an ordinary differential equation (ODE) in $\theta$. The set of solutions of this ODE, the IGO-flow, consists of continuous time models of the recovered algorithms in the limit of the population size going to infinity and the step-size (learning rate for ES or PBIL) to zero.

In this paper we analyze the convergence of the IGO-flow for isotropic ESs where the family of distributions is Gaussian with covariance matrix equal to an overall variance times the identity. The underlying algorithms are step-size adaptive ESs that resemble ESs with derandomized adaptation \cite{ostermeier1994derandomized} and encompass xNES \cite{Glasmachers2010gecco} and the pure rank-$\mu$ update CMA-ES with only one variance parameter \cite{Hansen2003ec}. Previous works have proposed and analyzed continuous models of ESs that are solutions of ODEs \cite{Yin:1995wj,Yin1995ec} using the machinery of stochastic approximation \cite{KushnerBOOK2003,Borkar:2008ts}. The ODE variable in these studies encodes solely the mean vector of the search distribution and the overall variance is taken to be proportional to $H(\nabla f)$ where $H$ is a smooth function with $H(0)=0$. Consequently the model analyzed looses invariance to monotonic transformation of the objective function and scale-invariance, both being fundamental properties of virtually all ESs. The technique relies on the Lyapunov function approach and assumes the stability of critical points of the ODE \cite{Yin:1995wj,Yin1995ec}. 
In this paper, our approach also relies on the stability of the critical points of the ODE that we analyze  by means of Lyapunov functions. However one difficulty stems from the fact that when convergence occurs, the variance typically converges to zero which is at the boundary of the definition domain $\Theta$. To circumvent this difficulty we extend the standard Lyapunov method to be able to study stability of boundary points.

Applying the extended Lyapunov's method to the IGO-flow in the manifold of isotropic Gaussian distributions, we derive a sufficient condition on the so-called weight function $\ww$---parameter of the algorithm and usually chosen by the algorithm designer---so that the IGO-flow converges to the global minimum independently of the starting point on objective functions that are composite of a monotonic function with a convex quadratic function. We will call those functions \emph{monotonic convex-quadratic-composite} in the sequel. We then extend this result to functions that are the composition of a monotonic transformation and a twice continuously differentiable function, called monotonic $\Ctwo$-composite in the rest of the paper. We prove local convergence to a local optimum of the function in the sense that starting close enough from a local optimum, with a small enough variance, the IGO-flow converges to this local optimum.

The rest of the paper is organized as follows. In Section~\ref{sec:esigo} we introduce the IGO-flow for the family of isotropic Gaussian distributions, which we call \textit{ES-IGO}-flow. In Section~\ref{sec:elst} we extend the standard Lyapunov's method for proving stability. In Section~\ref{sec:conv} we apply the extended method to the ES-IGO-flow and provide convergence results of the ES-IGO-flow on monotonic convex-quadratic-composite functions and on monotonic $\Ctwo$-composite functions.

\paragraph*{\bf Notation.} For $A \subset X$, where $X$ is a topological space, we let $A^{c}$ denote the complement of $A$ in $X$, $A^{o}$ the interior of $A$, $\overline A$ the closure of $A$, $\partial A = \overline A \setminus A^{o}$ the boundary of $A$. Let $\R$ and $\R^{d}$ be the sets of real numbers and $d$-dimensional real vectors, $\R_{\geq 0}$ and $\R_{+}$ denote the sets of non-negative and positive real numbers, respectively. Let $\norm{x}$ represent the Euclidean norm of $x \in \R^{d}$. The open and closed balls in $\R^{d}$ centered at $\theta$ with radius $r > 0$ are denoted by $B(\theta, r)$ and $\overline B(\theta, r)$.

Let $\leb$ denote the Lebesgue measure on either $\R$ or $\R^{d}$. Let $\PP{1}$ and $\PP{d}$ be the probability measures induced by the one-variate and $d$-variate standard normal distributions, $\pd{1}$ and $\pd{d}$ the probability density function induced by $\PP{1}$ and $\PP{d}$ w.r.t.\ $\leb$. 
Let $\pd{\theta}$ and $\PP{\theta}$ represent the probability density function w.r.t.\ $\leb$ and the probability measure induced by the Gaussian distribution $\Gauss(\mm(\theta), \CC(\theta))$ parameterized by $\theta \in \Theta$, where the mean vector $\mm(\theta)$ is in $\R^{d}$ and the covariance matrix $\CC(\theta)$ is a positive definite symmetric matrix of dimension $d$. We sometimes abbreviate $\mm(\theta(t))$ and $\CC(\theta(t))$ to $\mm(t)$ and $\CC(t)$. Let $\vect: \R^{d \times d} \to \R^{d^{2}}$ denote the vectorization operator such that $\vect: \CC \mapsto [\CC_{1,1}, \CC_{1,2},\dots, \CC_{1,d},\CC_{2,1},\dots, \CC_{d,d}]^\tp$, where $\CC_{i,j}$ is the $i,j$-th element of $\CC$. We use both notations: $\theta = [\mm^\tp, \vect(\CC)^\tp]^\tp$ and $\theta = (\mm, \CC)$.

\section{The ES-IGO-flow}\label{sec:esigo}

The IGO framework for continuous optimization with the family of Gaussian distributions is as follows. The original objective is to minimize an objective function $f: \R^{d} \to \R$. This objective function is mapped into a function on $\Theta$. Hereunder, we suppose that $f$ is $\leb$-measurable. Let $\ww: [0, 1] \to \R$ be a bounded, non-increasing weight function. We define the weighted quantile function \cite{Arnold2011arxiv} as
\begin{equation}\label{eq:w}
\Wf(x) = \ww\bigl( \PP{\theta}[y: f(y) \leq f(x)] \bigr)\enspace. 
\end{equation}
The function $\Wf(x)$ is a preference weight for $x$ according to the $\PP{\theta}$-quantile. The fitness value of $\thetad$ given $\theta$ is defined as the expectation of the preference $\Wf$ over $\PP{\thetad}$, 
$\Ft(\thetad) = \E_{x \sim P_\theta}\bigl[\Wf(x) \bigr]$. 
Note that since $\Wf(x)$ depends on $\theta$ so does $\Ft(\thetad)$. The function $\Ft$ is defined on a statistical manifold $(\Theta, \FM)$ equipped with the Fisher metric $\FM$ as a Riemannian metric. The Fisher metric is the natural metric. It is compatible with relative entropy and with KL-divergence and is the only metric that does not depend on the chosen parametrization. Using log-likelihood trick and exchanging the order of differentiation and integration, the ``vanilla'' gradient of $\Ft$ at $\thetad = \theta$ can be expressed as 
$\nabla_{\thetad} \Ft(\thetad) \rvert_{\thetad = \theta} = \E_{x \sim P_\theta}\bigl[ \Wf(x) \nabla_{\theta} \ln(\pd{\theta}(x)) \bigr]$. 
The natural gradient, that is, the gradient taken w.r.t.\ the Fisher metric, is given by the product of the inverse of the Fisher information matrix $\FIM$ at $\theta$ and the vanilla gradient, namely $\FIM^{-1} \nabla_{\thetad} \Ft(\thetad) \rvert_{\thetad = \theta}$. 
The IGO ordinary differential equation is defined as
\begin{equation}
\ddt{\theta} = \FIM^{-1} \nabla_{\thetad} \Ft(\thetad) \bigr\rvert_{\thetad = \theta}\enspace. \label{eq:igo}
\end{equation}
Since the right-hand side (RHS) of the above ODE is independent of $t$ the IGO ODE is autonomous. The IGO-flow is the set of solution trajectories of the above ODE \eqref{eq:igo}.

When the parameter $\theta$ encodes the mean vector and the covariance matrix of the gaussian distribution in the following way $\theta = [\mm^\tp, \vect(\CC)^\tp]^\tp$, the product of the inverse of the Fisher information matrix $\FIM^{-1}$ and the gradient of the log-likelihood $\nabla_{\theta} \ln (\pd{\theta}(x))$ can be written in an explicit form \cite{Akimoto-alg} and \eqref{eq:igo} reduces to
\begin{equation}
\ddt{\theta} = \int \Wf(x)
\begin{bmatrix}
x - \mm\\
\vect\bigl((x - \mm)(x - \mm)^\tp - \CC\bigr)
\end{bmatrix}
\PP{\theta}(\dx)\enspace.\label{eq:cma-igo}
\end{equation}
The pure rank-$\mu$ update CMA-ES \cite{Hansen2003ec} can be considered as an Euler scheme for solving \eqref{eq:cma-igo} with a Monte-Carlo approximation of the integral. Let $x_{1}, \dots, x_{n}$ be samples independently generated from $\PP{\theta}$. Then, the quantile $\PP{\theta}[y: f(y) \leq f(x_{i})]$ in \eqref{eq:w} is approximated by the number of solutions better than $x_{i}$ divided by $n$, i.e., $\bigl\lvert \{x_{j}, j = 1, \dots, n: f(x_{j}) \leq f(x_{i})\} \bigr\rvert / n =: R_{i} / n $. Then $\Wf(x_{i})$ is approximated by \nnew{$\ww\bigl( (R_{i} - 1/2) / n \bigr)$}\del{$\ww(R_{i} / n)$}\niko{wouldn't $\ww(R_i-1/2)/n$ be better?}\yohe{to make it consistent with IGO paper I followed Niko's suggestion}, where $\ww$ is the given weight function. The Euler scheme for approximating the solutions of \eqref{eq:cma-igo} where the integral is approximated by Monte-Carlo leads to
\begin{equation}
\theta^{t+1} = \theta^{t} + \eta \sum_{i=1}^{n} \nnew{\frac{\ww\bigl( (R_{i} - 1/2) / n \bigr)}{n}}\del{\frac{\ww(R_{i} / n)}{n}}
\begin{bmatrix}
x_{i} - \mm^{t}\\
\vect\bigl((x_{i} - \mm^{t})(x_{i} - \mm^{t})^\tp - \CC^{t} \bigr)
\end{bmatrix} \enspace,
\label{eq:cma-igo-alg}
\end{equation}
where $\eta$ is the time discretization step-size. This equation is equivalent to the pure rank-$\mu$ update CMA-ES when the learning rates $\eta_\mm$ \new{and $\eta_\CC$, for the update of $\mm^t$ and $\CC^t$ respectively,} are set to the same value $\eta$, while they \new{have}\del{are chosen to be} different values in practice ($\eta_\mm = 1$ and $\eta_\CC \leq 1$). The summation on the RHS in \eqref{eq:cma-igo-alg} converges to the RHS of \eqref{eq:cma-igo} with probability one as $\lambda \to \infty$ (Theorem~4 in \cite{Arnold2011arxiv}).

In the following, we study the simplified IGO-flow where the covariance matrix is parameterized by only a single \new{variance} parameter $\vv$ as $\CC = \vv I_{d}$\del{ where $\vv$ denotes the single variance of the Gaussian distribution, sometimes called the overall variance}. Under the parameterization $\theta = [\mm^\tp, \vv]^\tp$, \eqref{eq:igo} reduces to $\frac{\rmd \theta}{\rmd t} = \int \Wf(x) \bigl[\begin{smallmatrix} x - \mm \\ \norm{x - \mm}^{2}/d - \vv \end{smallmatrix}\bigr] \PP{\theta}(\dx)$. 
Using the change of variable $z = (x - \mm)/\sqrt{\vv}$, the above ODE reads
\begin{equation}\label{eq:esigo}
\ddt{\theta} = \gtt(\theta) 
\ ,\quad 
\gtt(\theta) = 
\int \Wf(\mm + \sqrt{\vv} z)
	\begin{bmatrix}
	\sqrt{\vv} z\\
	\vv (\norm{z}^{2}/d - 1 )
	\end{bmatrix}
P_{d}(\rmd z)
\end{equation}
and we rewrite it by part
\begin{align}
\textstyle \ddt{\mm} &= \gmm(\theta) 
  \ ,\quad 
\textstyle \gmm(\theta) = \sqrt{\vv} \int \Wf(\mm + \sqrt{\vv} z) z  \PPZ(d z) 
 \label{eq:esigo-m}\\
\textstyle \ddt{\vv} &= \gvv(\theta) 
  \ ,\quad 
\textstyle \gvv(\theta) = \vv \int \Wf(\mm + \sqrt{\vv} z) (\norm{z}^{2}/d  - 1 ) \PPZ(d z) \enspace. \label{eq:esigo-c} 
\end{align} 
The domain of this ODE is $\Theta = \{\theta = (\mm, \vv) \in \R^{d} \times \R_{+}\}$. We call \eqesigo the \textit{ES-IGO} ordinary differential equation. The following proposition shows that for a Lipschitz continuous \new{weight function $\ww$}, solutions of the ODE \eqref{eq:esigo} exist for any \new{initial condition $\theta(0) \in \Theta$} and are unique.
\begin{proposition}[Existence and Uniqueness]\label{prop:uni}
Suppose $\ww$ is Lipschitz continuous. Then the initial value problem: $\ddt{\theta} = \gtt(\theta)$, $\theta(0) = \theta_{0}$, has a unique solution on $[0, \infty)$ for each $\theta_{0} \in \Theta$, i.e.\ there is only one solution $\theta: \R_{\geq 0} \to \Theta$ to the initial value problem. 
\end{proposition}
\begin{proof}
We can obtain a lower bound $a(t) > 0$ and an upper bound $b(t) < \infty$ for $\vv(t)$ for each $t \geq 0$ under a bounded $\ww$. Similarly, we can have an upper bound $c(t) < \infty$ for $\norm{\mm(t)}$. Then we have that $(\mm(t), \vv(t)) \in E(t) = \{x \in \R^{d}: \norm{x} \leq c(t)\} \times \{x \in \R_{+}: a(t) \leq x \leq b(t)\}$ and $E(t)$ is compact for each $t \geq 0$. Meanwhile, $\gtt$ is locally Lipschitz continuous for a Lipschitz continuous $\ww$. Since $E(t)$ is compact, the restriction of $\gtt$ into $E(t)$ is Lipschitz continuous. Applying Theorem~3.2 in \cite{Khalil:2002wj} that is an extension of the theorem known as Picard-Lindel{\"o}f theorem or Cauchy-Lipschitz theorem, we have the existence and uniqueness of the solution on each bounded interval $[0, t]$. Since $t$ is arbitrary, we have the proposition.
\qed\end{proof}


Now that we know that solutions of the ES-IGO ODE exist and are unique, we define the ES-IGO-flow as the mapping $\flow: \R_{\geq 0} \times \Theta \to \Theta$, which maps $(t, \theta_{0})$ to the solution $\theta(t)$ of \eqref{eq:esigo} with initial condition $\theta(0) = \theta_{0}$. Note that we can extend the domain of $\gtt$ from $\Theta = \R^{d}\times \R_{+}$ to $\overline{\Theta} = \R^{d} \times \R_{\geq 0}$. It is easy to see from \eqref{eq:esigo} that the value of $\gtt(\theta)$ at $\theta = (\mm, 0)$ is $0$ for any $\mm \in \R^{d}$. However, we exclude the boundary $\partial \Theta$ from the domain for reasons that will become clear in the next section. Because the initial variance must be positive and the variance starting from positive region never reach the boundary in finite time, solutions $\flow(t,\cdot)$ will stay in the domain $\Theta$. However, as we will see, they can converge asymptotically towards points of the boundary. 

\del{Note that }Since $\Ft$ is {\em adaptive}, i.e.\ $J_{\theta_1}(\theta) \neq J_{\theta_2}(\theta)$ for $\theta_1 \neq \theta_2$ in general, it is not trivial to determine whether the solutions to \eqref{eq:igo} converge to points where $\gtt(\theta) = 0$\del{ or not}\footnote{If $\Ft$ is not adaptive and defined to be the expectation of the objective function $f(x)$ over $P_\theta$, convergence to the zeros of the RHS of \eqref{eq:igo} is easily obtained. For example, see Theorem~12 and its proof in \cite{Malago:2011hd}, where the solution to the system of a similar ODE whose RHS is the vanilla gradient of the expected objective function is derived and the convergence of the solution trajectory to the critical point of the expected function is proven.}. Even\del{ if we have chances to} know\new{ing} that they converge to zeros of $\gtt(\theta)$\del{, it} is not helpful at all, because \new{we have} $\gtt(\theta) = 0$ for any $\theta$ with variance zero and we are actually interested in convergence to the point $(\xstar, 0)$ where $\xstar$ is a local optimum of $f$.


\begin{remark}\label{rem:inv-param}
Because of the invariance property of the natural gradient, \del{under re-parameterization of the Gaussian distributions}the mean vector $\mm(\theta)$ and the variance $\vv(\theta)$ obey \eqref{eq:esigo-m} and \eqref{eq:esigo-c} \new{under re-parameterization of the Gaussian distributions}.
%
\del{
\begin{equation*}
\ddt{}\begin{bmatrix} \mm(\theta) \\ \vv(\theta) \end{bmatrix} %
= \int \Wf(\mm(\theta) + \sqrt{\vv(\theta)} z)
\begin{bmatrix}
\sqrt{\vv(\theta)} z\\
\vv(\theta) (\norm{z}^{2}/d - 1 )
\end{bmatrix}
\PPZ(\dz)\enspace.
\end{equation*}
}%
Therefore, the trajectories of $\mm$ and $\vv$ are also \new{independent of the}\del{ invariant under any} parameterization. For instance, we obtain the same trajectories $\vv(\theta)$ for any of the following parameterizations: $\theta_{d+1}=\vv$, $\theta_{d+1}=\sqrt{\vv}$, and $\theta_{d+1} = \frac12 \ln \vv$, although the trajectories of the parameters $\theta_{d+1}$ are of course different. Consequently, the same convergence results for $\mm(\theta)$ and $\vv(\theta)$ (see Section~\ref{sec:conv}) will hold under any parameterization. Parameterizations $\theta = (\mm, \vv)$ and $\theta = (\mm, \frac12 \ln \vv)$ correspond to the pure rank-$\mu$ update CMA-ES and the xNES with only one variance parameter. Thus, the continuous model to be analyzed encompasses both algorithms.
\end{remark}

\begin{remark}
Theory of stochastic approximation says that a stochastic algorithm $\theta^{t+1} = \theta^t + \eta h^t$ follows the solution trajectories of the ODE $\ddt{\theta} = \E[h^t \mid \theta^t = \theta]$ in the limit for $\eta$ to \nnew{zero}\del{$\infty$} under several conditions. In our setting, $\theta$ encodes $\mm$ and $\vv$ and the noisy observation $h^t = \sum_{i=1}^{\lambda} \ww_{R_i} \FIM^{-1} \nabla_\theta \ln p_{\theta^{t}}(x_i)$, where $\ww_{i}$, $i = 1, \dots, \lambda$, are predefined weights and $R_i$ is the ranking of $x_i$. If we define $\ww(p) = \sum_{i=1}^{\lambda} \ww_{i} \binom{\lambda-1}{i-1} p^{i-1}(1 - p)^{\lambda-i}$ in \eqref{eq:w}, then $\gtt(\theta) = \E[h^t \mid \theta^t = \theta]$ and the ODE agrees with \eqref{eq:esigo}. Therefore, \eqref{eq:esigo} can be viewed as the limit behavior of adaptive-ES algorithms not only in the case $\eta \to 0$ and $\lambda \to \infty$ but also in the case $\eta \to 0$ and \new{finite} $\lambda$\del{ remains finite}. Indeed, it is possible to bound the difference between $\{\theta^t, t \geq 0\}$ and the solution $\theta(\cdot)$ of the ODE \eqref{eq:esigo} by extending Lemma~1 in Chapter\nnew{~9} of \cite{Borkar:2008ts}.%
\del{Theory of stochastic approximation says that under several conditions a stochastic recursive algorithm $\theta^{t+1} = \theta^t + \eta h^t$ can be viewed as an noisy discretization of the associated ODE $\ddt{\theta} = H(\theta)$, where $H(\theta) = \E[h^t \mid \theta^t = \theta]$. In our setting, $\theta$ encodes $\mm$ and $\vv$ and the noisy observation $h^t = \sum_{i=1}^{\lambda} \ww_{R_i} \FIM^{-1} \nabla_\theta \ln p_{\theta^{t}}(x_i)$, \anne{$t$ or $\theta^{t}$ is missing in the RHS I guess} where $\ww_{i}$, $i = 1, \dots, \lambda$, are predefined weights and $R_i$ is the ranking of $x_i$ defined above. Then, $H(\theta)$ can be written in the same form as $\gtt(\theta)$ with weight function $\ww(p) = \sum_{i=1}^{\lambda} \ww_{i} \binom{\lambda-1}{i-1}  p^{i-1}(1 - p)^{\lambda-i}$. Hence, if the regularity conditions are satisfied we still have the ES-IGO ODE \eqref{eq:esigo} as the limit for $\eta$ to $\infty$. Indeed, it is possible to evaluate the difference between $\theta^t$ and $\theta(t \cdot \eta)$ in expectation by extending Lemma~1 in Chapter of \cite{Borkar:2008ts}.}%
 The details are omitted due to the space limitation.\footnote{When $H(\theta)$ is a (natural) gradient of a function, the stochastic algorithm is called a stochastic gradient method. The theory of stochastic gradient method (e.g., \cite{Bonnabel:2011to}) relates the convergence of the stochastic algorithm with the zeros of $H(\theta)$. However, it is not applicable to our algorithm due to the reason mentioned above Remark~\ref{rem:inv-param}.}
\end{remark}
\anne{an assumption is needed such that $\E[h^t \mid \theta^t = \theta]$ is a function of $\theta$. Maybe it is OK the way you wrote it since you start by ``Theory of stochastic approximation says that under several conditions'' but maybe it can be improved...}
\anne{Youhei, I see more or less what the 3 last sentences refer to but could you try to improve their formulation as it is not clear for the moment.}

\section{Extension of Lyapunov Stability Theorem}\label{sec:elst}

When convergence occurs, the variance typically converges to zero. Hence the study of the convergence of the solutions of the ODE will be carried out by analyzing the stability of the points $\stheta=(\xstar,0)$. However, because points with variance zero are excluded from the domain $\Theta$, we need to extend classical definitions of stability to be able to handle points located on the boundary of $\Theta$.

\begin{definition}[Stability]
Consider the following system of differential equation
\begin{equation}
\textstyle \dot \theta = \vf(\theta), \quad \theta(0) = \theta_0 \in D , \label{eq:auto}
\end{equation}%
where $\vf: D \mapsto \R^{\dd}$ is a continuous map and $D \subset \R^{\dd}$ is open. Then $\stheta \in \overline{D}$ is called 
\begin{itemize}
\item \textit{stable in the sense of Lyapunov}\footnote{%
Usually,\del{ the} stability is defined for stationary points. However, it is not the only case that a point is stable in our definition. Let $\stheta \in \overline{D}$ be a stable point. If $\stheta \in D$ or $\vf$ can be prolonged by continuity at $\stheta$ as $\lim_{\theta \to \stheta} \vf(\theta) = \vf(\stheta)$, then $\vf(\stheta) = 0$. That is, $\stheta$ is a stationary point. However, $\lim_{\theta \to \stheta} \vf(\theta)$ does not always exist for a stable boundary point $\stheta \in \partial D$. For example, consider the ODE: $\rmd \theta_1 / \rmd t = - \theta_1 / \sqrt{\theta_1^2 + \theta_2^2}$, $\rmd \theta_2 / \rmd t = - \theta_2$. The domain is $\R \times \R_+$. Then, $\abs{\theta_{1}}$ and $\theta_2$ are monotonically decreasing to zero. Hence, $(0,0)$ is globally asymptotically stable. However, $\lim_{\theta\to(0,0)} \vf(\theta)$ does not exist.} %
if for any $\epsilon > 0$ there is $\delta > 0$ such that $\theta_{0} \in D \cap \overline{B}(\stheta, \delta) \Longrightarrow \theta(t) \in D \cap \overline{B}(\stheta, \epsilon)$ for all $t \geq 0$, where $t \mapsto \theta(t)$ is any solution of \eqref{eq:auto};

\item \textit{locally attractive} if there is $\delta > 0$ such that $\theta_0 \in D \cap \overline{B}(\stheta, \delta) \Longrightarrow \lim_{t\to\infty} \lVert \theta(t) - \stheta \rVert = 0$ for any solution $t \mapsto \theta(t)$ of \eqref{eq:auto};

\item \textit{globally attractive} if $\lim_{t\to\infty} \norm{\theta(t) - \stheta} = 0$ for any $\theta_0 \in D$ and any solution $t \mapsto \theta(t)$ of \eqref{eq:auto};

\item \textit{locally asymptotically stable} if it is stable and locally attractive;

\item \textit{globally asymptotically stable} if it is stable and globally attractive.
\end{itemize}
\end{definition}

We can now understand why we need to exclude points with variance zero from the domain $\Theta$. Indeed, points with variance zero are points from where solutions of the ODE will never move because $\gtt( \theta ) = 0$\del{ for such points}. Consequently, if we include points $(x,0)$ in $\Theta$, none of these points can be attractive as in a neighborhood we always find $\theta_{0}=(x_{0},0)$ such that a solution starting in $\theta_{0}$ stays there and cannot thus converge to any other point. 
\yohe{The remark has been moved into a footnote.}

A standard technique to prove stability is\del{ to use the so-called} Lyapunov's method that consists in finding a scalar function $\VV: \R^\dd \to \R_{\geq 0}$ that is positive except for a candidate stable point 
$\stheta$ \new{with}\del{ to be studied for which} $\VV(\stheta) = 0$, and that is monotonically decreasing along any trajectory of the ODE. Such a function is called \textit{Lyapunov function} \new{(and is analogous to a potential function in dynamical systems)}. 
\del{The advantage of }Lyapunov's method\del{ is that it} does not require the analysis of the solutions of the ODE. The standard Lyapunov's stability theorem gives practical conditions to verify that a function is indeed a Lyapunov function. However, because our candidate stable points are located on $\partial \Theta$, we need to extend this standard \del{Lyapunov's }theorem.

\begin{lemma}[Extended Lyapunov Stability Method
]\label{lem:lyapunov-boundary}
Consider the autonomous system \eqref{eq:auto}, 
where $\vf: D \to \R^{\dd}$ is a map and $D \subset \R^{\dd}$ is the open domain of $\theta$. Let $\stheta \in \overline{D}$ be a candidate stable point. Suppose that there is an $R > 0$ such that \\
\asm{A1}: $\vf(\theta)$ is continuous on $D \cap B(\stheta, R)$;\\ 
\asm{A2}: there is a continuously differentiable $\VV: \R^{\new{\dd}} \to \R$ such that for some strictly increasing continuous function $\alpha: \R_+ \to \R_{+}$ satisfying $\lim_{p \to \infty} \alpha(p) = \infty$,
\begin{gather}
 V(\stheta) = 0, \quad V(\theta) \geq \alpha(\norm{\theta - \stheta}) \quad \forall \theta \in D \cap B(\stheta, R) \setminus\{\stheta\}\label{eq:lyap-p} \\
\text{ and } \qquad \nabla \VV(\theta)^\tp \vf(\theta) < 0 \quad \forall \theta \in D \cap B(\stheta, R) \setminus \{\stheta\} ; \label{eq:lyap-n}
\end{gather}
\asm{A3}: for any $r_1$ and $r_2$ such that $0 < r_1 \leq r_2 < R$, if a solution $\theta(\cdot)$ to \eqref{eq:auto} starting from $D_{r_1, r_2} = \{\theta \in D: r_{1} \leq \norm{\theta - \stheta} \leq r_{2}\}$ stays in $D_{r_1,r_2}$ for $t \in [0, \infty)$, then there is a $T \geq 0$ and a \emph{compact} set $E \subset D_{r_1,r_2}$ such that $\theta(t) \in E$ for $t \in [T, \infty)$. \\
\indent Then, $\stheta$ is locally asymptotically stable. If \asm{A1} and \asm{A2} hold with $D$ replacing $D \cap B(\stheta, R)$ and \asm{A3} holds \new{with $R = \infty$}\del{for any $r_1$ and $r_2$ such that $0 < r_1 \leq r_2$}, then $\stheta$ is globally asymptotically stable.
\end{lemma}
\begin{proof}
We follow \del{the line of }the proof of Theorem~4.1 in \cite{Khalil:2002wj}. \del{Just as in the proof of Theorem~4.1 in \cite{Khalil:2002wj}, we}\new{We} have from assumptions \asm{A1} and \asm{A2} that there is $\delta < R$ such that $\stheta$ is stable and $\VV(\theta(t)) \to \tVV \geq 0$ for each $\theta_{0} \in D \cap B(\stheta, \delta)$. Moreover, under \asm{A1} and \asm{A2} with $D$ replacing $D \cap B(\stheta, R)$ we have that $\VV(\theta(t)) \to \tVV \geq 0$ for each $\theta_{0} \in D$. Since $\lim_{t \to \infty} \VV(\theta(t)) \to 0$ implies $\lim_{t \to \infty} \norm{\theta - \stheta} = 0$ by \eqref{eq:lyap-p}, it is enough to show $\tVV = 0$. We show $\tVV = 0$ by contradiction argument. Assume that $\tVV > 0$. Then, \del{as is proved in the proof of Theorem~4.1 in \cite{Khalil:2002wj} }we have that for each $\theta_{0} \in D$ (or $\in D \cap B(\stheta, \delta)$ for the case of local asymptotic stability) there are $r_{1}$ and $r_{2}$ such that $0 < r_{1}\leq r_{2}$ ($\leq \delta$) and $\theta(t)$ lies in $D_{r_{1}, r_{2}}$ for $t \geq 0$. Note that $D_{r_1,r_2}$ is not necessarily a compact set. This is different from Theorem~4.1 in \cite{Khalil:2002wj}. By assumption \asm{A3} we have that there is a compact set $E$ and $T \geq 0$ such that $\theta(t) \in E$ for $t \geq T$. Since $\VV$ is continuously differentiable and $\vf$ is continuous, $\nabla \VV(\theta)^\tp \vf(\theta)$ is continuous. Then, the function $\theta \mapsto \VV(\theta)^\tp \vf(\theta)$ has its maximum $-\beta$ on the compact $E$ and $- \beta < 0$ by \eqref{eq:lyap-n}. This leads to $\VV(\theta(t)) \leq \VV(\theta(T)) - \beta (t - T) \downarrow - \infty$ as $t \to \infty$. This contradicts the hypothesis that $\VV > 0$. Hence, $\tVV = 0$ for any $\theta_{0} \in D$ (or $\in D \cap B(\stheta, \delta)$). 
\qed\end{proof}
\section{Convergence of the ES-IGO-flow}\label{sec:conv}



In this section we study the convergence properties of the ES-IGO-flow $\flow: (t, \theta_{0}) \mapsto \theta(t)$, where $\theta(\cdot)$ represents the solution to the ES-IGO ODE \eqref{eq:esigo} with initial value $\theta(0) = \theta_{0}$, i.e., $\ddt{\flow(t, \theta_{0})} = \gtt(\flow(t, \theta_{0}))$ and $\flow(0, \theta_{0}) = \theta_{0}$. By the definition of asymptotic stability, the global asymptotic stability of $\stheta \in \overline{\Theta}$ implies the global convergence, that is, $\lim_{t \to \infty} \flow(t, \theta_0) = \stheta$ for all $\theta_0 \in \Theta$. Moreover, the local asymptotic stability of $\stheta \in \overline{\Theta}$ implies the local convergence, that is, $\exists \delta > 0$ such that $\lim_{t \to \infty} \flow(t, \theta_0) = \stheta$ for all $\theta_0 \in \Theta \cap B(\stheta, \delta)$.
We will prove convergence properties of the ES-IGO-flow by applying Lemma~\ref{lem:lyapunov-boundary}. In order to prove our result we need to make the following assumption on $\ww$:\\
\asm{B1}: $\ww$ is non-increasing and Lipschitz continuous with $\ww(0) > \ww(1)$; \\
\asm{B2}: $\int \ww(\PP{1}[y: y \leq z]) (z^{2}/d-1/d) \PP{1}(\dz) = \alpha > 0$.

Assumption \asm{B1} is not restrictive. Indeed, the non-increasing and non-constant property of $\ww(\cdot)$ is a natural requirement and any weight setting in \eqref{eq:cma-igo-alg} can be expressed, for any given population size $n$, as a discretization of some Lipschitz continuous weight function. Assumption \asm{B2} is satisfied if and only if the variance $\vv$ diverges exponentially on a linear function. In fact, $\gvv(\theta)$ defined in \eqref{eq:esigo-c} reduces to $\vv \int \ww(\PP{1}[y: y \leq z]) (z^{2}/d-1/d) \PP{1}(\dz)$ when $f(x) = a^\tp x$ for $\forall a \in \R^d \setminus \{0\}$ and we have that $\dot \vv = \alpha \vv$ and the solution is $\vv(t) = \vv_{0} \exp(\alpha t)$. Then, $\vv(t) \to \infty$ as $t \to \infty$. Assumption \asm{B2} holds, for example, if $\ww$ is convex and not linear.

Let $\sif$ be the set of strictly increasing functions $g: \R \to \R$ that are $\leb$-measurable and $\Ctwo$ be the set of twice continuously differentiable functions $h: \R^{d} \to \R$ that are $\leb$-measurable. Under \asm{B1} and \asm{B2}, we have the following main theorems.

\begin{theorem}
\label{thm:quad}
Suppose that the objective function $f$ is a monotonic convex-quadratic-composite function $g \circ h$, where $g \in \sif$ and $h$ is a convex quadratic function $x \mapsto (x - \xstar)^\tp A (x - \xstar) / 2$ where $A$ is positive definite and symmetric. Assume that \asm{B1} and \asm{B2} hold. Then, $\stheta = (\xstar, 0) \in \overline{\Theta}$ is the globally asymptotically stable point of the ES-IGO. Hence, we have the global convergence of $\flow(t, \theta_0)$ to $\stheta$.
\end{theorem}
\begin{proof}
Since the ES-IGO does not explicitly utilize the function values but uses the quantile $P_{\theta}[y: f(y) \leq f(x)]$ which is equivalent to $P_{\theta}[y: g^{-1} \circ f(y) \leq g^{-1} \circ f(x)]$, without loss of generality we assume $f = h$.

According to Lemma~\ref{lem:lyapunov-boundary}, it is enough to show that \asm{A1} and \asm{A2} hold with $D (= \Theta)$ replacing $D \cap B(\stheta, R)$ and \asm{A3} holds \new{with $R = \infty$}\del{for any $r_1$ and $r_2$ such that $0 < r_1 \leq r_2$}. As is mentioned in the proof of Proposition~\ref{prop:uni}, $\gtt$ is locally Lipschitz continuous for a Lipschitz continuous $\ww$. Thus, \asm{A1} is satisfied under \asm{B1}.

 We can choose as a Lyapunov candidate function $\VV(\theta) = \sum_{i=1}^{d} (\mm_{i} - \xstar_{i})^{2} + d \cdot \vv = \norm{\mm - \xstar}^2 + \Tr(\vv I_d)$. All the conditions on $\VV$ described in \asm{A2} are obvious except for the negativeness of $\nabla \VV(\theta)^\tp \gtt(\theta)$. To show the negativeness, rewrite $\gtt(\theta)$ as $\int \Wf(\mm + \sqrt{\vv} z)\gttz P_d(\dz)$. The idea is to show the (strictly) negative correlation between $\Wf(\mm + \sqrt{\vv} z)$ and $\nabla \VV(\theta)^\tp \gttz$ by using an extension of the result in \cite[Chapter~1]{Thorisson:2000uo} and apply the inequality $\int \Wf(\mm + \sqrt{\vv} z) \nabla \VV(\theta)^\tp \gttz P_d(\dz) < \int \Wf(\mm + \sqrt{\vv} z)P_d(\dz) \int \nabla \VV(\theta)^\tp \gttz P_d(\dz)$ = 0. We use the non-increasing property of $\ww$ with $\ww(0) > \ww(1)$ in \asm{B1} to show the negative correlation.

To prove \asm{A3}, we require \asm{B2}. Since a continuously differentiable function can be approximated by a linear function at any non-critical point $\xbar$, the natural gradient $\gtt$ is approximated by that on a linear function in a small neighborhood of (\xbar, 0). We use the property $\leb[x: f(x) = \bar f] = 0$ to approximate $\gtt$. As is mentioned above, \asm{B2} implies $\gvv$ on a linear function is positive. By using the approximation and this property, we can show that $E = D_{r_1,r_2} \cap \{\theta: \vv \geq \bar \vv \}$ satisfies \asm{A3} for some $\bar \vv > 0$.
\qed\end{proof}
We have that for any initial condition $\theta(0)=(\mm_{0}, \vv_{0})$, the search distribution $\PP{\theta}$ weakly converges to the Dirac measure $\delta_{\xstar}$ concentrated at the global minimum point $\xstar$. This result is generalized to monotonic $\Ctwo$-composite functions using a quadratic Taylor approximation. However, global convergence becomes local convergence.

\begin{theorem}
\label{thm:general}
Suppose that the objective function $f$ is a monotonic $\Ctwo$-composite function $g \circ h$, where $g \in \sif$ and $h \in \Ctwo$ has the property that $\leb[x: h(x) = s] = 0$ for any $s \in \R$. Assume that \asm{B1} and \asm{B2} hold. Let $\xstar$ be a critical point of $h$, i.e.\ $\nabla h(\xstar) = 0$, with a positive definite Hessian matrix $A$. Then, $\stheta = (\xstar, 0) \in \overline{\Theta}$ is a locally asymptotically stable point of the ES-IGO. Hence, we have the local convergence of $\flow(t, \theta_0)$ to $\stheta$. Moreover, if $\xbar$ is not a critical point of $h(\cdot)$, for any $\theta_{0} \in \Theta$, $\flow(t,\theta_{0})$ will never converge to $\btheta = (\xbar, 0)$.
\end{theorem}
\begin{proof}
\del{The proof is similar to the proof of Theorem~\ref{thm:quad}. }As in the proof of Theorem~\ref{thm:quad}, we assume $f = h$ without loss of generality. The proofs of \asm{A1} and \asm{A3} carry over from Theorem~\ref{thm:quad} because we only use\new{d} the property $\leb[x: f(x) = \bar f] = 0$. To show \asm{A2}, we use the Taylor approximation of the objective function $f$. Since \del{the objective function}\new{f} is approximated by a quadratic function in a neighborhood of a critical point $\xstar$, we approximate the natural gradient by the corresponding natural gradient on the quadratic function. Then, employing the same Lyapunov candidate function as in the previous theorem we can show \asm{A2}. Because of the approximation, we only have local asymptotic stability. The last statement of Theorem~\ref{thm:general} is an immediate consequence of the approximation of the natural gradient and \asm{B2}.
\qed\end{proof}
We have that starting from a point close enough to a local minimum point $\xstar$ with a sufficiently small initial variance, the search distribution weakly converges to $\delta_{\xstar}$. It is not guaranteed for the parameter to converge somewhere when the initial mean is not close enough to the local optimum or the initial variance is not small enough. Theorem~\ref{thm:general} also states that the convergence $(\mm(t), \vv(t)) \to (\xbar, 0)$ does not happen for $\xbar$ such that $\nabla h(\xbar) \neq 0$. That is, the continuous time ES-IGO does not prematurely converge on a slope of the landscape of $f$.

\section{Conclusion}\label{sec:disc}

In this paper we have proven the local convergence of the continuous time model associated to step-size adaptive ESs towards local minima on monotonic $\Ctwo$-composite functions. In the case of monotonic convex-quadratic-composite functions we have proven the global convergence, i.e. convergence independently of the initial condition (provid\new{ed the initial} step-size is strictly positive) towards the unique minimum. Our analysis relies on investigating the stability of critical points associated to the underlying ODE that follows from the Information Geometric Optimization setting. We use a classical method for the analysis of stability of critical points, based on Lyapunov functions. We have however extended the method to be able to handle convergence towards solutions at the boundary of the ODE definition domain. We believe that our approach is general enough to handle more difficult cases like the CMA-ES with a more general covariance matrix. We want to emphasize that the model we have analyzed is the correct model for step-size \emph{adaptive} ESs as the ODE encodes both the mean vector \emph{and} step-size and preserves fundamental invariance properties of the algorithm.

\section*{Acknowledgments}
{\small
This work was partially supported by the ANR-2010-COSI-002
grant (SIMINOLE) of the French National Research Agency and the ANR COSINUS project ANR-08-COSI-007-12.
}
%
%
\small

\begin{thebibliography}{10}

\bibitem{Auger2005tcs}
Auger, A.:
\newblock {Convergence results for the (\(1\), \(\lambda\))-SA-ES using the
  theory of \(\varphi\)-irreducible Markov chains}.
\newblock Theoretical Computer Science \textbf{334}(1-3) (2005)  35--69

\bibitem{Jagerskupper:2006ur}
J{\"a}gersk{\"u}pper, J.:
\newblock {Probabilistic runtime analysis of (\(1+,\lambda\)), ES using
  isotropic mutations}.
\newblock In: Proceedings of the 2006 Genetic and Evolutionary Computation
  Conference -- GECCO 2006, ACM (2006)  461--468

\bibitem{Jagerskupper:2006cf}
J{\"a}gersk{\"u}pper, J.:
\newblock {How the (\(1+1\)) ES using isotropic mutations minimizes positive
  definite quadratic forms}.
\newblock Theoretical Computer Science \textbf{361}(1) (2006)  38--56

\bibitem{Jagerskupper2007tcs}
J{\"a}gersk{\"u}pper, J.:
\newblock {Algorithmic analysis of a basic evolutionary algorithm for
  continuous optimization}.
\newblock Theoretical Computer Science \textbf{379}(3) (2007)  329--347

\bibitem{Arnold2011arxiv}
Arnold, L., Auger, A., Hansen, N., Ollivier, Y.:
\newblock {Information-geometric optimization algorithms: a unifying picture
  via invariance principles}.
\newblock arXiv:1106.3708v1 (2011)

\bibitem{Hansen2003ec}
Hansen, N., Muller, S.D., Koumoutsakos, P.:
\newblock {Reducing the time complexity of the derandomized evolution strategy
  with covariance matrix adaptation (CMA-ES)}.
\newblock Evolutionary Computation \textbf{11}(1) (2003)  1--18

\bibitem{Chandy:1995ux}
Baluja, S., Caruana, R.:
\newblock {Removing the genetics from the standard genetic algorithm}.
\newblock In: Proceedings of the 12th International Conference on Machine Learning. (1995)

\bibitem{ostermeier1994derandomized}
Ostermeier, A., Gawelczyk, A., Hansen, N.:
\newblock A derandomized approach to self-adaptation of evolution strategies.
\newblock Evolutionary Computation \textbf{2}(4) (1994)  369--380

\bibitem{Glasmachers2010gecco}
Glasmachers, T., Schaul, T., Yi, S., Wierstra, D., Schmidhuber, J.:
\newblock Exponential natural evolution strategies.
\newblock In: Proceedings of Genetic and Evolutionary Computation Conference,
  ACM (2010)  393--400


\bibitem{Yin:1995wj}
Yin, G.G., Rudolph, G., Schwefel, H.P.:
\newblock {Establishing connections between evolutionary algorithms and
  stochastic approximation}.
\newblock Informatica \textbf{1} (1995)  93--116

\bibitem{Yin1995ec}
Yin, G.G., Rudolph, G., Schwefel, H.P.:
\newblock {Analyzing the (\(1\), \(\lambda\)) evolution strategy via stochastic
  approximation methods}.
\newblock Evolutionary Computation \textbf{3}(4) (1996)  473--489

\bibitem{KushnerBOOK2003}
Kushner, H.J., Yin, G.G.:
\newblock {Stochastic approximation and recursive algorithms and applications}.
  2nd edn.
\newblock Springer Verlag (2003)

\bibitem{Malago:2011hd}
Malag{\`o}, L., Matteucci, M., Pistone, G.:
\newblock {Towards the geometry of estimation of distribution algorithms based on the exponential family}.
\newblock In: Proceedings of Foundations of Genetic Algorithms (FOGA '11),
  ACM (2011)  230--242

\bibitem{Akimoto-alg}
Akimoto, Y., Nagata, Y., Ono, I., Kobayashi, S.:
\newblock {Theoretical foundation for CMA-ES from information geometry
  perspective}.
\newblock Algorithmica, Online First (2011)

\bibitem{Khalil:2002wj}
Khalil, H.K.:
\newblock {Nonlinear systems}.
\newblock Prentice-Hall, Inc. (2002)

\bibitem{Borkar:2008ts}
Borkar, V.S.:
\newblock {Stochastic approximation: a dynamical systems viewpoint}.
\newblock Cambridge University Press (2008)

\bibitem{Bonnabel:2011to}
Bonnabel, S.:
\newblock {Stochastic gradient descent on Riemannian manifolds}.
\newblock arXiv:1111.5280v2 (2011)


\bibitem{Thorisson:2000uo}
Thorisson, H.:
\newblock {Coupling, stationarity, and regeneration}.
\newblock Springer Verlag (2000)

\end{thebibliography}

\end{document}